\documentclass[letterpaper, 10 pt, journal, twoside]{ieeeconf}  

\IEEEoverridecommandlockouts                              

\usepackage{amsmath} 
\usepackage{amsthm}
\usepackage{algorithm}
\usepackage[noend]{algpseudocode}
\usepackage{amssymb}  
\usepackage{xcolor}
\usepackage{graphicx}
\usepackage{booktabs}
\usepackage{hyperref}
\usepackage{qting}
\usepackage{cite}
\usepackage{scalerel}

\DeclareMathOperator*{\maximize}{maximize}				
\DeclareMathOperator*{\st}{subject\,to}					
\newcommand{\mat}[1]{\ensuremath{\begin{bmatrix}#1\end{bmatrix}}}	

\newcommand{\dq}[0]{\ensuremath{\dot{q}}}					

\newcommand{\Xs}[0]{\ensuremath{\mathcal{X}}}						
\newcommand{\Us}[0]{\ensuremath{\mathcal{U}}}						
\newcommand{\Vs}[0]{\ensuremath{\mathcal{V}}}						

\newcommand{\DX}[0]{\ensuremath{\partial \mathcal{X}}}						
\newcommand{\DV}[0]{\ensuremath{\partial \mathcal{V}}}						
\newcommand{\DDQ}[0]{\ensuremath{\partial \dot{\mathcal{Q}}}}						
\newcommand{\intV}[0]{\ensuremath{\mathrm{int}(\mathcal{V})}}					

\newcommand{\T}[0]{\ensuremath{\top}}							
\newcommand{\pinv}[0]{\ensuremath{\dagger}}					
\newcommand{\Rv}[1]{\ensuremath{\mathbb{R}^{#1}}}				
\newcommand{\R}[2]{\ensuremath{\mathbb{R}^{#1\times #2}}}		
\newcommand{\Expect}{{\rm I\kern-.3em E}}				

\newtheorem{theorem}{Theorem}
\newtheorem{lemma}{Lemma}
\newtheorem{corollary}{Corollary}
\newtheorem{assumption}{Assumption}

\makeatletter
\renewenvironment{proof}[1][\proofname]{\par
\pushQED{\qed}%
\normalfont \topsep6\p@\@plus6\p@\relax
\trivlist
\item\relax
{\itshape
#1\@addpunct{.}}\hspace\labelsep\ignorespaces
}{%
\popQED\endtrivlist\@endpefalse
}
\makeatother

\newcommand{\rev}[1]{#1}

\newcommand{\asia}[1]{#1}

\title{\LARGE \bf
VBOC: Learning the Viability Boundary of a Robot Manipulator using Optimal Control
}

\author{Asia La Rocca$^{1}$, Matteo Saveriano$^{1}$ and Andrea Del Prete$^{1}$
\thanks{Manuscript received: April 28, 2023; Revised: July 31, 2023; Accepted: August 23, 2023.}
\thanks{This paper was recommended for publication by Editor Jens Kober upon evaluation of the Associate Editor and Reviewers’ comments.}
\thanks{$^{1}$The authors are with the Industrial Engineering Department, University of Trento, Via Sommarive 11, 38123, Trento, Italy. {\tt \{name.surname\}@unitn.it}}%
\thanks{This work has been partially supported by the PRIN project DOCEAT (CUP n. E63C22000410001) and the European Union under the NextGenerationEU project iNest (ECS 00000043).}
\thanks{Digital Object Identifier (DOI): see top of this page.}
}

\begin{document}

\maketitle

\begin{abstract}
Safety is often the most important requirement in robotics applications. 
Nonetheless, control techniques that can provide safety guarantees are still extremely rare for nonlinear systems, such as robot manipulators. 
A well-known tool to ensure safety is the viability kernel, which is the largest set of states from which safety can be ensured.
Unfortunately, computing such a set for a nonlinear system is extremely challenging in general. 
Several numerical algorithms for approximating it have been proposed in the literature, but they suffer from the curse of dimensionality.
This paper presents a new approach for numerically approximating the viability kernel of robot manipulators. 
Our approach solves optimal control problems to compute states that are guaranteed to be on the boundary of the set. 
This allows us to learn directly the set boundary, therefore learning in a smaller dimensional space. 
Compared to the state of the art on systems up to dimension 6, our algorithm resulted to be more than 2 times as accurate for the same computation time, or 6 times as fast to reach the same accuracy.
\end{abstract}

\section{INTRODUCTION}

The computation of viability kernels is a topic of great importance in the field of safe control of constrained dynamical systems. The viability kernel is the set of states from which a dynamical system can remain within a predefined set of safe states. Knowing the viability kernel, it is straightforward to design safe controllers; it is therefore a very powerful tool for safety-critical applications.
Unfortunately, except for the case of linear dynamics and linear constraints~\cite{Blanchini1999, DelPrete2018}, computing these sets is extremely challenging.

The classical method for the computation of these sets, known as the Viability Kernel Algorithm~\cite{Saintpierre1994}, consists in gridding the state space and approximating the viability kernel using recursive inclusions. Due to the grid-based discretization of the state space that is performed, the complexity and memory requirements of this algorithm scale exponentially with respect to the state dimension (problem known as the curse of dimensionality). The problem has then been approached with many different tools such as viability theory~\cite{Aubin1991}, theory of barriers~\cite{Russwurm2021}, approximate dynamic programming~\cite{Coquelin2007}, or simulated annealing~\cite{Bonneuil2006}. 

One of the most relevant directions is given by the approaches based on reachability analysis, exploiting the connection between viability kernels and reachable sets~\cite{Mitchell2007}. Reachability analysis consists in inferring the set of all states that are backward/forward reachable by a constrained dynamical system from a given target/initial set of states. For nonlinear systems, examples are given in~\cite{Monnet2016,Bravo2003,Zanolli2020}, using interval arithmetic, or in~\cite{Coquelin2007,Jiang2016,Rubies2016}, using dynamic programming (level set approaches based on the solution of the Hamilton-Jacobi PDE). These algorithms try to accurately define the set boundary, but they suffer from the curse of dimensionality, restricting their application to small systems.

In recent years, an improvement in the field has been brought by function approximators, such as Neural Networks (NNs), which allow to represent complex sets in a more memory-efficient way with respect to gridding~\cite{Jiang2016,Rubies2016,Djeridane2006}. These algorithms require less memory to run and store the resulting approximation, which represented one of the main bottlenecks of previous approaches.

\qting{quote:qlearningref}{\rev{Another recent trend consists in the use of Reinforcement Learning, as in~\cite{Hsu2021}, where the authors exploit Q-learning to compute reach-avoid sets. This method allows to compute a safe under-approximation of these sets, but its applicability is still limited due to the need to discretize the action space.}}

A promising approach consists then in data-driven methods that rely on the feasibility of Optimal Control Problems (OCPs). In~\cite{Zhou2020} the authors approximated the forward invariant region of nonlinear systems using Support Vector Machines (SVMs). Their algorithm solves OCPs for initial conditions well distributed over the state space, and uses the feasibility results to train a classifier. This approach is applicable to a wide range of nonlinear systems, but it becomes intractable for large state spaces, since it requires too many samples to get good approximations. 

A promising way to reduce the computational burden of data-driven methods is Active Learning (AL), a machine learning framework to iteratively select new data samples that are the most informative or representative. In a first study~\cite{Chapel2006}, the authors proposed an iterative algorithm to select the points nearest to the frontier of the learned SVM classifier to improve its accuracy. However, the approach is still hardly scalable because the number of samples in large dimensions can still grow exponentially; moreover, nonlinear SVMs training complexity scales more than quadratically with the number of samples.



Overall, the computation of viability kernels remains an active area of research, with ongoing efforts to develop more efficient and scalable algorithms. 
In this paper we propose a new approach for the approximation of viability kernels of robot manipulators. Instead of computing the set boundary by iteratively approaching it (as in Active Learning or Hamilton-Jacobi methods), we directly compute states that are \emph{exactly} on the boundary. We then use these states to train an NN that approximates the set. The main advantage of this approach is that it requires significantly less samples than other data-driven methods. 
While our approach is tailored to robot manipulators (or any fully-actuated multi-body system), many of our theoretical results hold for any smooth dynamical system.
Our tests on systems with 2, 4 and 6 dimensional states show that it leads to faster and more accurate approximations than state-of-the-art approaches.

\section{PRELIMINARIES}

\subsection{Notation} \label{ssec:notation}
\begin{itemize}
    \item \DV\ denotes the boundary of the set \Vs ;
    \item \intV\ denotes the interior of the set \Vs ;
    \item $\Xs \setminus \Vs$ denotes the set difference between \Xs\ and \Vs ; 
    \item $\{ x_i \}_0^N$ denotes a discrete-time trajectory given by the sequence $(x_0, \dots, x_N)$;
    \item \rev{$x^+$ denotes the next state, whenever $x$ is used to denote the current state.}
\end{itemize}

\subsection{Problem statement}
Let us consider a discrete-time dynamical system with state and control constraints:
\begin{equation}
    x_{i+1} = f(x_i, u_i),
\qquad
    x \in \Xs \rev{\subset \Rv{n}}, \qquad u \in \Us \rev{\subset \Rv{m},}
\end{equation}
\qting{quote:X_and_U}{\rev{where \Xs\ and \Us\ are the closed and bounded sets of feasible states and control inputs.}
}
Our goal is to compute a numerical approximation of the \emph{viability kernel} \Vs , which is the subset of \Xs\ starting from which it is possible to keep the state in \Xs\ indefinitely. Mathematically, we can define \Vs\ as:
\begin{equation}
\Vs \triangleq \{ x_0 \in \Xs \, | \, \exists \{u_i\}_{0}^{\infty}: x_i \in \Xs, u_i \in \Us, \forall \, i=0,\dots,\infty \} .
\end{equation}
\rev{In the following, we assume that \Vs\ is closed.}

\begin{assumption} \label{ass:dyn_differentiability}
We assume that $f(\cdot)$ be differentiable with respect to $x$, which implies:
\begin{equation}
|| \mathrm{eig}(\partial_x f(x, u)) || < \infty \qquad \forall \, x \in \Xs, u \in \Us,
\end{equation}
where the function $\mathrm{eig}(\cdot)$ returns the eigenvalues of the given matrix.
\end{assumption}
\qting{quote:differentiability}{
\rev{Assumption~\ref{ass:dyn_differentiability} implies that our method cannot handle nonsmooth systems, such as a robot that makes contact with a perfectly rigid environment. 
However, contacts can sometimes be modeled as visco-elastic, recovering then differentiability.}} 

\subsection{Backward Reachability vs Viability}
\label{ssec:back_reach_vs_viab}
Our goal is to approximate \Vs , which is the largest control-invariant set~\cite{Blanchini1999} (i.e., a set inside which you can remain indefinitely). 
However, for many applications (e.g., ensuring recursive feasibility of MPC~\cite{Hewing2020a}, or designing a safety filter for a Reinforcement Learning algorithm~\cite{Wabersich2019}) we could settle for just any sufficiently large control-invariant set.
Control-invariant sets can be computed using N-step backward reachability, i.e., computing the set of points from which a given set can be reached in $N$ steps.
In certain cases, backward reachability can also be used to compute \Vs , and in the following we assume that this is the case.
\begin{assumption} \label{ass:back_reach}
Let us define $\mathcal{S}$ as the set containing all the equilibrium states of our system:
\begin{equation}
    \mathcal{S} = \{ x \in \Xs\ | \, \exists \, u \in \Us : x = f(x,u) \}.
\end{equation}
We assume that the $\infty$-step backward reachable set of $\mathcal{S}$ is equivalent to \Vs .
In other words, that a state is viable if and only if from that state you can reach an equilibrium state.
\end{assumption}
We argue that this assumption is satisfied for most robot manipulators, which are our main focus. 
However, even if this assumption were not satisfied, our approach could still be used to compute control-invariant sets. 

\subsection{\rev{Data-Driven Learning}}
\label{ssec:active_learning}
\rev{A state-of-the-art approach to numerically approximate backward reachable sets is to sample states $x^{sample} \in \Xs$ and verify whether, from there, it is possible to reach the target set~\cite{Djeridane2008}, $\mathcal{S}$ for our application.} This can be done by solving an OCP like the following one:
\begin{equation} \label{eq:ocp_classic}
\begin{aligned}
    \maximize_{ \{x_i\}_{0}^N, \{u_i\}_{0}^{N-1}} \, & 1 \\
    \st \,\, & x_{i+1} = f(x_i, u_i) \quad \forall \, i=0,\dots,N-1 \\
    & x_i \in \Xs, \, u_i \in \Us    \quad \forall \, i=0,\dots,N-1 \\
    & x_0 = x^{sample} \\
    & x_N = x_{N-1}
\end{aligned}
\end{equation}
where $N \in \mathbb{N}$ is the time horizon, which must be sufficiently large to allow the system to reach, if possible, an equilibrium state from $x^{sample}$. 
If a solution of this problem is found, then we know that $x^{sample} \in \Vs$, and the whole state trajectory is in \Vs . 
Otherwise, we can assume that $x^{sample} \not \in \Vs$, even though this is not necessarily the case because a solution may exist even if the solver was unable to find one. While potentially impactful, this issue is typically neglected by assuming that the solver can find a solution if one exists.
This information is then used to train a classifier (e.g., SVM or NN) to distinguish viable and non-viable states.

This approach scales badly because it requires a dense sampling of \Xs\ to accurately approximate \Vs . To reduce complexity, it has been coupled with Active Learning (AL), a technique to choose the most informative or representative values of $x^{sample}$. \rev{A first study on the application of AL for the computation of viability kernels was done in~\cite{Chapel2006}, where the authors proposed an algorithm based on iteratively testing the nearest points to the frontier of the currently learned SVM classifier. 
More advanced AL algorithms for reachable sets approximations can be found in~\cite{Chakrabarty2019,Chakrabarty2020}.}

Instead of \rev{ iteratively approaching \DV}, the next section presents an approach to directly compute states that are \emph{exactly} on \DV .

\section{Viability-Boundary Optimal Control (VBOC)} 
\label{sec:vboc}
To find states that are \emph{exactly} on \DV , we solve a modified version of OCP~\eqref{eq:ocp_classic}, where the initial state $x_0$ is not completely fixed, but it is optimized through a cost function:
\begin{equation} \label{eq:viab_boundary_ocp}
\begin{aligned}
    \maximize_{\{x_i\}_{0}^N, \{u_i\}_{0}^{N-1}} \, & a^\T x_0 \\
    \st \,\, & x_{i+1} = f(x_i, u_i) \quad \forall \, i=0,\dots,N-1 \\
    & x_i \in \Xs, \, u_i \in \Us    \quad \forall \, i=0,\dots,N-1 \\
    & S x_0 = s \\
    & x_N = x_{N-1},
\end{aligned}
\end{equation}
where $a \in \Rv{n}$ is the cost vector, $S \in \R{n_s}{n}$ and $s \in \Rv{n_s}$ are \rev{the initial} constraint matrix and vector, \rev{with $n_s$ being their size}. 
\qting{quote:explain_S}{
\rev{While the role of $a$ is straightforward, the use of $S$ and $s$ to partially constrain $x_0$ will become clear in Section~\ref{ssec:data_distribution}.}
}
\qting{quote:explain_lemma_1}{Let us now prove that the initial state of any locally-optimal solution of~\eqref{eq:viab_boundary_ocp} is on \DV . }

\begin{lemma} \label{lemma:viab_boundary_ocp}
Let us consider a locally-optimal state trajectory $\{x_i^* \}_{0}^N$ computed by solving~\eqref{eq:viab_boundary_ocp}. 
Let us assume that $P_S a \neq 0$, where $P_S \triangleq (I - S^\pinv S)$ is a null-space projector of $S$. 
Then, if $N$ is sufficiently large to allow reaching $\mathcal{S}$ from any viable state, we have: 
$$
x_0^* \in \DV .
$$
Moreover, for any sufficiently small value $\epsilon > 0$:
$$
\tilde{x}_0 \triangleq x_0^* + \epsilon P_S a \notin \Vs .
$$
\end{lemma}

\begin{proof}
We split this proof in two cases: when $\tilde{x}_0 \notin \Xs$ and when $\tilde{x}_0 \in \Xs$.
In the first case, $\tilde{x}_0 \notin \Xs$ implies $\tilde{x}_0 \notin \Vs$. 
Moreover, we know by definition of~\eqref{eq:viab_boundary_ocp} that $x_0^* \in \Vs$.
Since $\tilde{x}_0$ and $x_0^*$ can be arbitrarily close, then we can infer $x_0^* \in \DV$.
In the second case ($\tilde{x}_0 \in \Xs$), we can prove this lemma by contradiction. We suppose that $x_0^* \not \in \DV$, which implies $\tilde{x}_0 \in \Vs$ for a sufficiently small $\epsilon$, and we show that this leads to the conclusion that $x_0^*$ is not a local optimum.
If $x_0^* \not \in \DV$ then we know that $x_0^* \in \intV$. 
Together with the fact that $\tilde{x}_0 \in \Xs$, this means that $\tilde{x}_0 \in \Vs$ for any sufficiently small $\epsilon >0$. It is easy to verify that $\tilde{x}_0$ satisfies the initial conditions of~\eqref{eq:viab_boundary_ocp}, indeed:
\begin{equation}
    S \tilde{x}_0 = S x_0^* + \epsilon S P_S a = S x_0^* = s,
\end{equation}
where we have exploited the fact that $S P_S = 0$.
If $\tilde{x}_0$ is viable, by the assumption that $N$ be sufficiently large, it must be possible to satisfy also the terminal conditions of~\eqref{eq:viab_boundary_ocp}, i.e., to reach an equilibrium state. Finally, $\tilde{x}_0$ gives a better cost for~\eqref{eq:viab_boundary_ocp} than $x_0^*$ because:
\begin{equation}
    a^\T \tilde{x}_0 = a^\T x_0^* + \epsilon a^\T P_S a > a^\T x_0^*,
\end{equation}
where we have exploited the fact that $a^\T P_S a > 0$ because all null-space projectors (as $P_S$) are positive semi-definite and $P_S a \neq 0$ by assumption. 
In conclusion, since using $\tilde{x}_0$ as initial state it is possible to satisfy all the constraints of~\eqref{eq:viab_boundary_ocp}, while achieving a better cost, this implies that $\{x_i^* \}_{0}^N$ be not a local optimum. Therefore, if $\{x_i^* \}_{0}^N$ is a local optimum, it must hold that $x_0^* \in \DV$ and $\tilde{x}_0 \notin \Vs$.
\end{proof}

\qting{quote:explain_lemma_2}{Lemma~\ref{lemma:viab_boundary_ocp} ensures that using problem~\eqref{eq:viab_boundary_ocp} gives us trajectories that always start from \DV . 
However, the remaining $N$ states (from $x_1^*$ to $x_N^*$) could belong to \intV.
Ideally, we would like to compute trajectories that are entirely on \DV .
While this is not guaranteed, we argue that often some parts of $\{ x_i^* \}_1^N$ are on \DV , and we provide a simple method to check when this is the case.
To this aim, we start by showing that, under certain conditions, a viable state trajectory that starts on \DV , remains on \DV\ as long as these conditions are met.}

\begin{lemma} \label{lemma:boundary_V}
Given a state $x\in \DV$, a control $u \in \Us$, and a state direction $d \in \Rv{n}, ||d||=1$, such that for any sufficiently small $\epsilon > 0$:
\begin{equation}
\label{eq:ass_lemma_V_boundary}
\tilde{x} \triangleq x + \epsilon \, d \in \Xs \setminus \Vs,
\end{equation}
then we have that:
\begin{equation}
    x^+ = f(x, u) \not\in \intV .
\end{equation}
\end{lemma}

\begin{proof}
Since $\tilde{x}$ is in \Xs\ but not in \Vs, any state reachable from $\tilde{x}$ cannot be in \Vs ; therefore we can write:
\begin{equation}
\label{eq:proof_lemma_V_boundary}
\begin{aligned}
    \tilde{x}^+ & = f(\tilde{x}, u) = \\
    & = f(x,u) + \epsilon \, \partial_x f(x,u) d + O(\epsilon^2) = \\
    & = x^+ + \epsilon \, \partial_x f(x,u) d + O(\epsilon^2) \not \in \Vs .
\end{aligned}
\end{equation}
Since $\epsilon$ can be arbitrarily close to zero, and the eigenvalues of $\partial_x f(x, u)$ are bounded (Assumption~\ref{ass:dyn_differentiability}), this implies that $\tilde{x}^+$ can be arbitrarily close to $x^+$. Since $\tilde{x}^+ \not \in \Vs$, we can infer that $x^+$ can either be outside \Vs , or on its boundary, but not in its interior.
\end{proof}

\qting{quote:corollary_1_motivation}{\rev{To better understand assumption~\eqref{eq:ass_lemma_V_boundary}, let us introduce a Corollary,} which is a special case of Lemma~\ref{lemma:boundary_V}. 
}
\qting{quote:explain_corollary_1}{\rev{This Corollary states that if a trajectory starts on \DV , it cannot reach \intV\ before reaching \DX . 
}}

\begin{corollary} \label{corollary:boundary_V}
Given \rev{a control $u \in \Us$ and a state $x\in \DV$} satisfying the following assumption:
\begin{equation}
\label{eq:ass_corollary_boundary_V}
x \not\in \DX ,
\end{equation}
then we have that: $x^+ = f(x, u) \notin \intV$.
\end{corollary}
\begin{proof}
This corollary is a special case of Lemma~\ref{lemma:boundary_V} because~\eqref{eq:ass_corollary_boundary_V} implies~\eqref{eq:ass_lemma_V_boundary}. Indeed, if $x \in \DV$ and $x \not\in \DX$, then there must exist a direction $d \in \Rv{n}$ in which $x$ can be perturbed with an arbitrarily small magnitude $\epsilon$, so that it leaves \Vs\ without leaving \Xs , which is what~\eqref{eq:ass_lemma_V_boundary} states.
\end{proof}

Lemma~\ref{lemma:boundary_V} states something similar \rev{to Corollary~\ref{corollary:boundary_V}}, but clarifying that actually reaching \DX\ is necessary but not sufficient to reach \intV . The real condition to be met is indeed~\eqref{eq:ass_lemma_V_boundary}.
\qting{quote:explain_theorem_1}{\rev{The next Theorem exploits Lemma~\ref{lemma:boundary_V} to suggest} a simple method to verify whether the optimal states $\{ x_i^* \}_1^N$, computed by solving \eqref{eq:viab_boundary_ocp}, are on \DV .}

\begin{theorem} \label{theorem:boundary_V}
Let us consider a locally-optimal state trajectory $\{x_i^* \}_{0}^N$ computed by solving~\eqref{eq:viab_boundary_ocp}. 
Let us assume that $P_S a \neq 0$, where $P_S \triangleq (I - S^\pinv S)$ is a null-space projector of $S$. 
Let us assume that $N$ is sufficiently large to reach $\mathcal{S}$ from any state in \Vs .
Consider the following definition of a perturbed state trajectory:
\begin{equation}
\label{eq:ass_theorem_boundary_V_perturbed}
\begin{aligned}
\tilde{x}_0 &= x_0^* + \epsilon P_S a ,\\
\tilde{x}_{i+1} &= f(\tilde{x}_i, u_i^*) .
\end{aligned}
\end{equation}
Let us assume that for any sufficiently small $\epsilon > 0$ we have:
\begin{equation}
\label{eq:ass_theorem_boundary_V}
\tilde{x}_i \in \Xs \qquad i = 0, \dots, k-1 ,
\end{equation}
for a certain time step $k \in [0, N]$. Then we have:
\begin{equation}
    x_i^* \in \DV \qquad i=0, \dots, k . 
\end{equation}
\end{theorem}

\begin{proof}
The key idea of this proof is to iteratively apply Lemma~\ref{lemma:boundary_V} to show that $x_{i+1}^* \in \DV$, starting from the knowledge that $x_i^* \in \DV$, $u_i^* \in \Us$, and $\tilde{x}_i \in \Xs \setminus \Vs$.
We initialize the proof by exploiting Lemma~\ref{lemma:viab_boundary_ocp}, which states that $x_0^* \in \DV$ and $\tilde{x}_0 \notin \Vs$. 
Considering also assumption~\eqref{eq:ass_theorem_boundary_V} and the obvious fact that $u_i^* \in \Us$, $\forall \, i$, we have all the conditions to apply Lemma~\ref{lemma:boundary_V} for $i=0$.
Lemma~\ref{lemma:boundary_V} tells us only that $x_{i+1}^* \notin \intV$.
However, since all the optimal states must be viable by definition of~\eqref{eq:viab_boundary_ocp}, we can infer $x_{i+1}^* \in \DV$.
To iterate the application of Lemma~\ref{lemma:boundary_V} we need to show that $\tilde{x}_{i+1} \in \Xs \setminus \Vs$. 
Assumption~\eqref{eq:ass_theorem_boundary_V} ensures that $\tilde{x}_{i+1} \in \Xs$. 
The fact that $\tilde{x}_{i+1} \notin \Vs$ is instead a consequence of $\tilde{x}_i \in \Xs \setminus \Vs$ because, by definition of \Vs , we cannot reach \Vs\ from the outside without violating a constraint.
\end{proof}

Theorem~\ref{theorem:boundary_V} provides us with the theoretical foundations to design an iterative algorithm for numerically approximating \Vs{} in Section~\ref{sec:algorithm}. However, before we do that, the next section analyzes an interesting property of viability kernels of robotic manipulators, which we exploit to customize our algorithm.


\section{Viability for robot manipulators}
\qting{quote:dynamics_tensor}{
Let us introduce the dynamics of a robot manipulator with $n_j$ DOFs \rev{, using an unconventional form for the velocity term~\cite{Hollerbach1983_paper}}:
}
\begin{equation} \label{eq:manipulator_dynamics}
    M(q) \ddot{q} + \dot{q}^\T C(q) \dot{q} + g(q) = u,
\end{equation}
where $q, \dot{q}, \ddot{q} \in \Rv{n_j}$ are the joint positions, velocities, and accelerations, $M \in \R{n_j}{n_j}$ is the positive-definite mass matrix, 
\qting{quote:coriolis_tensor}{
$C(q)\in \R{n_j}{n_j\times n_j}$ is the 3D tensor accounting for Coriolis and centrifugal effects
}, 
and $g(q) \in \Rv{n_j}$ are the gravity torques.
We assume that $q$, $\dot{q}$, and $u$ are bounded:
\begin{equation} \label{eq:manipulator_constraints}
\begin{aligned}
q \in \mathcal{Q} &\triangleq \{q \in \Rv{n_j} | q^{min} \le q \le q^{max} \} , \\
\dot{q} \in \dot{\mathcal{Q}} &\triangleq \{\dot{q} \in \Rv{n_j} | \dot{q}^{min} \le \dot{q} \le \dot{q}^{max} \} , \\
u \in \Us &\triangleq \{u \in \Rv{n_j} | u^{min} \le u \le u^{max} \},
\end{aligned}
\end{equation}
where we assume that $\dot{q}^{min}<0$ and $\dot{q}^{max}>0$.
\begin{assumption} \label{ass:torque_limits}
    Let us assume that the robot is sufficiently strong to compensate for gravity in any configuration:
    \begin{equation} \label{eq:ass_torque_limits}
    g(q) \in \Us \quad \forall \, q \in \mathcal{Q}.
\end{equation}
\end{assumption}
\qting{quote:explain_lemma_3}{
In the following lemma we show that, for this class of systems, \Vs\ is \emph{star-convex} with respect to the joint velocities. In other words, if a state $(q, \dot{q}) \in \Vs$, then all states $(q, \alpha \, \dot{q}) \in \Vs$ for $\alpha \in [0,1]$.
}

\begin{lemma} \label{lemma:starredkernel}
Let us consider a manipulator with dynamics~\eqref{eq:manipulator_dynamics} and constraints~\eqref{eq:manipulator_constraints}. Under Assumption~\ref{ass:torque_limits}, its viability kernel is starred with respect to the joint velocities.
\end{lemma}

\begin{proof}
    If a state $(q_0, \dot{q}_0) \in \Vs$, it means there exists an infinite-time feasible trajectory starting with that state: $(q(t), \dot{q}(t)) \in \Xs$, $\forall \, t \ge 0$, with $q(0)=q_0$ and $\dot{q}(0)=\dot{q}_0$. 
    We now prove that all the states $(q_0, \alpha \dot{q}_0)$ are viable $\forall \, \alpha \in [0,1]$ by showing that the time-scaled trajectory $\tilde{q}(t) \triangleq q(\alpha t)$ is feasible. 
    To prove this, we exploit the time-scaling property of manipulator trajectories~\cite{Hollerbach1983_paper}.
    
    The time-scaled trajectory trivially satisfies the joint position and velocity limits, so we only need to prove that it also satisfies the control constraints.
    The time-scaled joint velocities and accelerations are:
    \begin{equation}
        \dot{\tilde{q}}(t) = \alpha \dot{q}(\alpha t), \quad \ddot{\tilde{q}}(t) = \alpha^2 \ddot{q}(\alpha t).
    \end{equation}
    Substituting these expressions in the dynamics~\eqref{eq:manipulator_dynamics} we get:
    \begin{equation}
        \alpha^2 ( M(q) \ddot{q} + \dot{q}^\T C(q) \dot{q} ) + g(q) = \tilde{u}(\alpha),
    \end{equation}
    where we expressed the control inputs $\tilde{u}$ as a function of $\alpha$. 
    Since the original trajectory $q(t)$ is feasible by assumption, we know that $\tilde{u}(1) \in \Us$.
    Moreover, by assumption~\eqref{eq:ass_torque_limits}, we know that $\tilde{u}(0) \in \Us$.
    Finally, by convexity of \Us , we can infer that $\tilde{u}(\alpha) \in \Us$, $\forall \, \alpha \in [0, 1]$, proving the feasibility of the time-scaled trajectory and the viability of $(q_0, \alpha \dot{q}_0)$.
\end{proof}


\subsection{Star-convex Viability Set Representation} 
\label{ssec:set_representation}
In general, \Vs\ can be encoded with a non-parametric classifier, such as a feedforward NN $\phi(x): \Rv{n} \rightarrow \Rv{}$, which takes as input a state $x$ and gives as output a binary label (viable, unviable).
To train such a classifier, both positive (viable) and negative (unviable) examples are needed. 
However, so far we have focused on computing viable states on \DV , which are therefore all positive examples. 
We could use the perturbed states $\tilde{x}_i(\epsilon)$ (described in Theorem~\ref{theorem:boundary_V}) as negative examples, but choosing the proper value of $\epsilon$ could be hard. 
A too small value could lead to positive and negative samples that are too close, making the training of the classifier extremely difficult.
On the other hand, a too large value could lead to poor classification accuracy.
\qting{quote:samples_on_dv}{\rev{To avoid these issues, we suggest to exploit that our samples are on $\DV$, and that $\Vs$ is star-convex (Lemma~\ref{lemma:starredkernel}) to encode \Vs\ differently.}}

Rather than using a classifier, we could encode \Vs\ with a function $\phi(q, d): \Rv{n_j} \times \Rv{n_j} \rightarrow \Rv{}$ that takes as inputs the joint positions $q$, the joint velocity direction $d$ (with $||d||=1$), and computes the maximum viable joint velocity norm.
In other words, if $\gamma = \phi(q, d)$, then $(q, \gamma d) \in \DV$.
With this representation of \Vs , we have transformed the classification problem into a regression problem and we no longer need unviable states to learn \Vs . 

\subsection{Uniform Data Distribution}
\label{ssec:data_distribution}
Solving instances of OCP~\eqref{eq:viab_boundary_ocp} we can compute viable trajectories that are guaranteed to start from \DV. However, to use these trajectories to learn \Vs, they need to cover its surface as uniformly as possible. This could be hard if \Vs\ is non-convex (which is in general the case), since just maximising $a^\T x_{0}$ for uniformly random directions $a$ would not ensure a \rev{complete} coverage of \DV\ (e.g., see Fig.~\ref{fig:star_convex_set}). In this case, the resulting initial state distribution would depend on the shape of the set and it could result in an accumulation of data on sharper areas of the set boundary and absence of data in other areas.
\begin{figure}[tbp]
    \centering
    \includegraphics[width=0.35\textwidth]{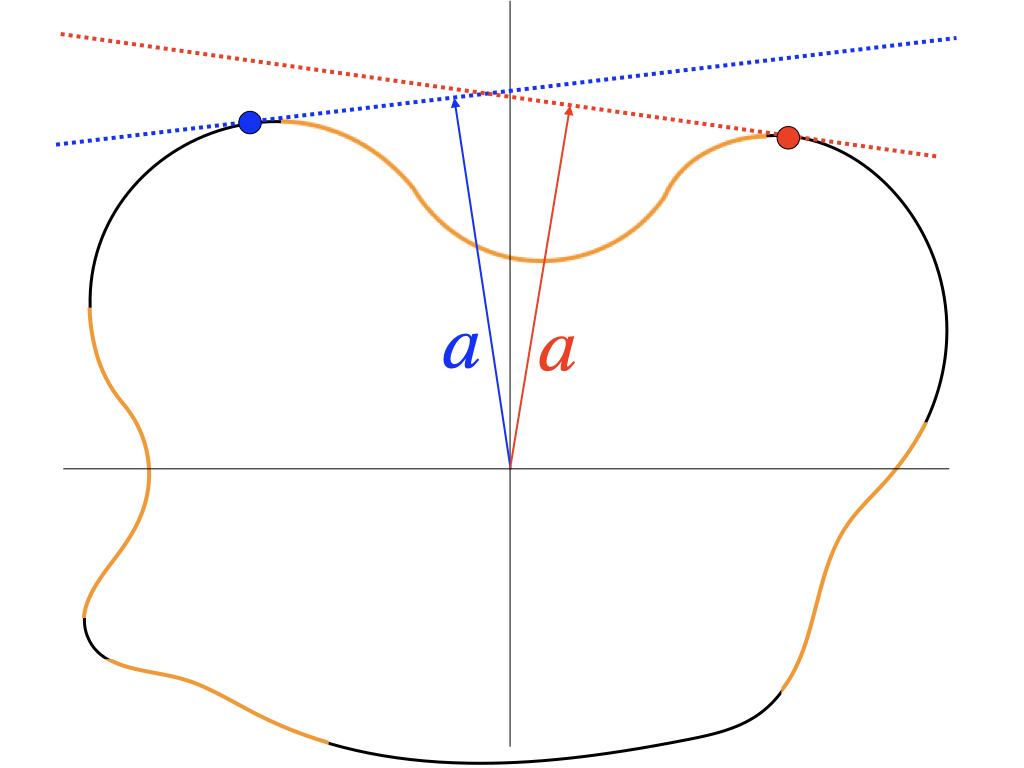}
    \caption{A star-convex set, with two possible examples of choices of $a$. The orange \rev{parts} of the set boundary cannot be discovered by any choice of $a$.}
    \label{fig:star_convex_set}
\end{figure}

\rev{Ensuring a uniform coverage of \DV\ does not seem possible without knowing its shape. However, our main concern is to ensure a uniform coverage of the input space of our set representation $\phi(\cdot)$, which is the space of joint angles and joint velocity directions. This is achieved by simply fixing} $q_0$ and the direction of $\dot{q}_0$ to uniformly random values $q^{init}$ and $d$, while maximizing $||\dot{q}_0||$. This is obtained by choosing:
\begin{equation}
\label{eq:compute_a_S_s}
\begin{aligned}
a = \mat{0 \\ d},
\quad
S = \mat{I & 0 \\ 0 & I - d d^\T},
\quad
s = \mat{q^{init} \\ 0} ,
\end{aligned}
\end{equation}
which leads to a cost that is orthogonal to $S$, ensuring the satisfaction of the assumption of Lemma~\ref{lemma:viab_boundary_ocp} ($P_S a \neq 0$).

This choice ensures a uniform distribution of the initial states. However, our method also exploits other states of the optimal trajectories to train the NN, whose distribution depends on the system dynamics. Using the strategy described above, we have observed an accumulation of data at lower velocities, where the trajectories converge to satisfy the terminal constraint. Empirically, we have observed that initializing one joint position at one of its bounds leads to a better coverage of $\DV$, because it allows for \rev{longer extreme trajectories}. The other joint positions are still uniformly randomized, as the initial joint velocity direction. Finally, to avoid trivial instances of the OCP, we ensure that the initial velocity direction of the joint that starts at its bound points away from it (e.g., if $q_0[i] = q_0^{max}[i]$ then $\dot{q}_0[i] <0$). 


\section{Algorithmic Implementation}
\label{sec:algorithm}
This section presents the implementation of our algorithm, summarized by the pseudo-code in Algorithm~\ref{alg:ocp} and~\ref{alg:vboc}.

Our approach is to generate trajectories that are, at least partially, on $\DV$. OCP~\eqref{eq:viab_boundary_ocp} returns indeed a trajectory that, even if only locally-optimal, is guaranteed to start from $\DV$. However, this is true only if the horizon $N$ is sufficiently long. To ensure this is the case, we solve~\eqref{eq:viab_boundary_ocp} with an increasing value for $N$, starting from a (reasonable) initial guess, until $a^\T x_0^*$ converges. Algorithm~\ref{alg:ocp} describes this procedure.

To well approximate $\Vs$ using these states we use the data generation approach described in Section~\ref{ssec:data_distribution} (lines 3-6 of Alg.~\ref{alg:vboc}). 

After solving OCP~\eqref{eq:viab_boundary_ocp} (line 7 of Alg.~\ref{alg:vboc}), we must check which optimal states belong to $\DV$ and can therefore be added to the dataset $\mathcal{D}$. By Lemma~\ref{lemma:viab_boundary_ocp}, we know that $x_0^* \in \DV$, and, therefore, we can add it to $\mathcal{D}$ (line 8). Moreover, we know that all viable states on \DX\ are also on \DV\ (simply because $\Vs \subseteq \Xs$), so we could add them to $\mathcal{D}$. However, because \Vs\ is starred with respect to \dq\ (see Section~\ref{ssec:set_representation}), we only add the states on \DDQ\ (line 9). At position limits there could be multiple states on $\DV$ with the same joint position and velocity direction, but different velocity norms, therefore these states would be conflicting in our starred-set representation and should be discarded.
To check if other states are on \DV, we exploit Theorem~\ref{theorem:boundary_V}. We compute the perturbed states $\tilde{x}_i$ as long as they belong to $\Xs$ (lines 10-13), and we add the associated optimal states to $\mathcal{D}$ (lines 14-15), discarding the states on \DX\, because they have already been considered (line 9). 

If the perturbed trajectory leaves \Xs\ at $\tilde{x}_{j-1}$, the associated optimal state $x^*_{j-1}$ must be on $\DX$ and, therefore, Lemma~\ref{lemma:boundary_V} tells us that the rest of the state trajectory could belong to $\intV$. To check if this is the case, we exploit again Lemma~\ref{lemma:starredkernel}, which tells us that the states on $\DV$, except for those at joint position limits, have maximum velocity norm for that position and velocity direction. So, when the optimal trajectory leaves $\DX$, say at $x^*_j \notin \DX$ (line 18), we solve another OCP~\eqref{eq:viab_boundary_ocp}, fixing the initial position and velocity direction to those of $x^*_j$ (lines 19-22). If this OCP returns the same initial velocity norm of $x^*_j$, then this proves that $x^*_j \in \DV$. If instead the OCP returns a higher initial velocity norm, it means that $x^*_j$ was in \intV, but it gives anyway a new trajectory starting from \DV\ that can be used in place of the previous one (lines 23-26). In both cases, $x^*_j$ can be added to $\mathcal{D}$ (line 27) and the whole process can continue.


\begin{algorithm}[t]
\caption{Viability-Boundary Optimal Control (VBOC)}
\small
\begin{algorithmic}[1]
\Require 
Constraint sets $\mathcal{X}$ and $\mathcal{U}$, 
Dynamics $f(\cdot, \cdot)$,
Number of DOFs $n_j$,
Time horizon $N_{start}$,
Time horizon increment $n$,
OCP~\eqref{eq:viab_boundary_ocp},
Initial joint positions $q^{init}$,
Initial joint velocity direction $d$

    \State $a, s, S \leftarrow \eqref{eq:compute_a_S_s}$
    
    \State $N, \gamma \leftarrow N_{start}, 0$
    \Repeat
        \State $\{x_i^*\}_{0}^N, \{u_i^*\}_{0}^{N-1} \leftarrow \textsc{OCP}(\mathcal{X}, \mathcal{U}, f, N, a, S, s)$
        \State $\gamma_{previous}, N \leftarrow \gamma, N+n$
        \State $\gamma \leftarrow a^\T x_0^*$
    \Until{$\gamma > \gamma_{previous}$}

\State $\textbf{return } \{x_i^*\}_{0}^N, \{u_i^*\}_{0}^{N-1}, N, a$

\end{algorithmic}
\label{alg:ocp}
\end{algorithm}

\begin{algorithm}[t]
\caption{Compute states on \DV}
\small
\begin{algorithmic}[1]

\Require 
Constraint sets $\mathcal{Q}$ and $ \mathcal{\dot{Q}}$, 
Number of DOFs $n_j$
Number of trajectories $K$,
Perturbation parameter $\epsilon$,
Time horizon $N_{guess}$

\State $\mathcal{D} \leftarrow []$
\For{$k = 0 \rightarrow K$}
    
    \State $q \leftarrow \textsc{RandomUniform}(q^{min}, q^{max})$
    \State $i \leftarrow \textsc{RandomInteger}(0, n_j)$
    \State $q[i] \leftarrow \textsc{RandomChoice}([q^{min}[i], q^{max}[i]])$
    \State $d \leftarrow \textsc{RandomVelocityDirection}(i, q)$
    \State $\{x_i^*\}_{0}^N, \{u_i^*\}_{0}^{N-1}, N, a \leftarrow \textsc{VBOC}(q, d, N_{guess})$
    
    \State $\textbf{insert } x_0^* \textbf{ in } \mathcal{D}$
    \For{\rev{$l = 1 \rightarrow N$}}
        \State \textbf{ if } $\dq_l^* \in \DDQ$ $\textbf{insert } x_l^* \textbf{ in } \mathcal{D}$
    \EndFor
    \State $\tilde{x}_0 \leftarrow x_0^* + \epsilon a$
    
    \For{$j = 1 \rightarrow N$}
        \If{$\tilde{x}_{j-1} \in \mathcal{X}$}
            \State $\tilde{x}_j \leftarrow f(\tilde{x}_{j-1}, u_{j-1}^*)$
            \State \textbf{if} $x_j^* \notin \DX$ $\textbf{insert } x_j^* \textbf{ in } \mathcal{D}$
        \Else
            \State $\tilde{x}_j \leftarrow \tilde{x}_{j-1}$
            \If{$x_j^* \notin \DX$}
                \State $\gamma \leftarrow a^\T x_j^*$
                \State $d \leftarrow \textsc{VelocityDirection}(x_j^*)$
                \State $q \leftarrow \textsc{JointPositions}(x_j^*)$
                \State $\{x_i^*\}_{j}^N, \{u_i^*\}_{j}^{N-1}, \sim, a \; \leftarrow \textsc{VBOC}(q, d, N-j)$
                \State $\gamma_{new} \leftarrow a^\T x_j^*$
                \If{$\gamma_{new} > \gamma$}
                    \For{$l = j+1 \rightarrow N$}
                        \State \textbf{if} $\dq_l^* \in \DDQ$ $\textbf{insert } x_l^* \textbf{ in } \mathcal{D}$
                    \EndFor
                \EndIf
                \State $\textbf{insert } x_j^* \textbf{ in } \mathcal{D}$
                \State $\tilde{x}_j \leftarrow x_j^* + \epsilon a$
            \EndIf
        \EndIf
    \EndFor
\EndFor
\Return $\mathcal{D}$
\end{algorithmic}
\label{alg:vboc}
\end{algorithm}

\section{Results}
To study the performance of our algorithm (VBOC) we test it with 2, 4 and 6-dimensional systems. 
\qting{quote:comparison}{
We compare VBOC with \rev{two state-of-the-art algorithms, focusing the comparison on the data-generation part, which is our main contribution.}}
\qting{quote:compare_algorithms}{\rev{The chosen algorithms are: i) the approach presented in Section~\ref{ssec:active_learning} relying on an informative-based Active Learning (AL) algorithm~\cite{Chakrabarty2020}}, and ii) a Hamilton-Jacoby Reachability (HJR) algorithm~\cite{Rubies2016}.} HJR is an approximate dynamic programming algorithm that computes the solution of the HJI PDE through recursive regression (since we are interested in infinite-time backward reachability, we discard the time dependency).

We evaluated the accuracy of the results by generating a test set using only the initial states obtained by calling Alg.~\ref{alg:ocp} with fully random initial position and velocity direction, to obtain \rev{well} distributed samples on $\DV$ \rev{(using the whole state trajectories would result in a higher density of samples at low velocities)}. On this set of $N$ points, we measured the Root Mean Squared Error (RMSE), defined as:
\begin{equation}
    \text{RMSE} = \sqrt{ \frac{1}{N}\sum_{i=0}^{N-1} \left( ||\dq_i|| - \phi \left(q_i, \frac{\dq_i}{|| \dq_i||} \right) \right)^2},
\end{equation}
where $\phi(\cdot, \cdot)$ is the NN trained by VBOC.
For AL and HJR, the trained NN is instead a classifier. Therefore, to measure the RMSE, we numerically identify (via binary search) the classifier boundary for the given joint positions $q_i$ and velocity direction $\dq_i / ||\dq_i||$.

In the tests we have run VBOC with $N_{start} = 100$, $\epsilon = 10^{-2}$ \rev{and solver tolerances equal to $10^{-3}$}. 

We have used fully-connected NNs composed of $3$ layers with ReLU activation functions. 
All algorithms are implemented\footnote{Our code is available at \url{github.com/idra-lab/VBOC}.} in Python, using ACADOS~\cite{Verschueren2019} for solving the OCPs and the PyTorch~\cite{Paszke2019} implementation of Adam~\cite{Kingma2014} for the NNs training. The tests are performed on a computer with 32 AMD\textregistered Ryzen9 5950x processors and a GeForce RTX 3060 GPU. The OCPs are solved in parallel on 30 cores and the NNs training is performed with CUDA on the GPU. 

\subsection{Tests on a 2D system}
The tested system is a simple pendulum, a model with a single swinging link connected to a fixed base through a revolute joint. The system has a 2-dimensional (2D) state space $x = [ q \; \dot{q} ]^\T$ and a 1-dimensional (1D) control input $u$. \rev{The joint positions, velocities and input constraints are in the form~\eqref{eq:manipulator_constraints} and are set to $\pi \pm \pi/4$ rad, $\pm 10$ rad/s, and $\pm 3$ Nm, respectively. }
With a 2D system, the computation of $\Vs$ is actually simpler than explained in Algorithm \ref{alg:vboc} since its boundary does not require sampling to be explored. It is sufficient to generate two trajectories with initial positions fixed at the two extremes $q^{min} / q^{max}$. 
The used NN has $100$ neurons in the hidden layer. The algorithm converged in $13\,$s and the resulting $\Vs$ approximation and the training data are shown in Fig.~\ref{fig:1p}. 

\begin{figure}[!tbp]
    \centering
    \includegraphics[width=0.49\textwidth]{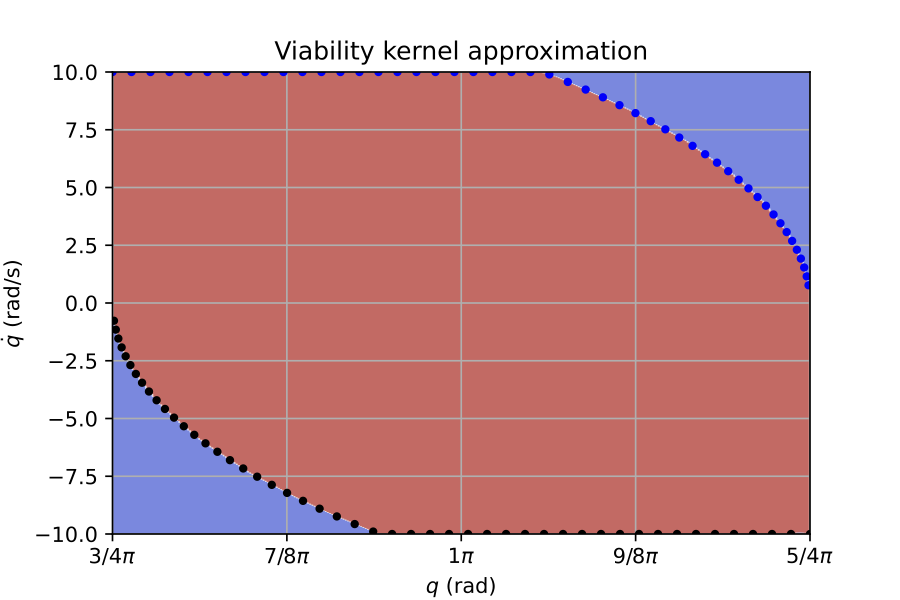}
    \caption{\rev{Viability kernel for the single pendulum. The background color represents the set learned using VBOC. The black and blue dots represent the training data from the two generated trajectories. The axis limits correspond to the joint position and velocity limits. }} 
    \label{fig:1p}
\end{figure}

To highlight the improvement with respect to the other state-of-the-art approaches, we compare the computational time and RMSE obtained using the same solver and NN complexity. AL and HJR have been executed on a grid with $100^2$ samples. Table~\ref{table:1prmse} shows the results. VBOC results to be faster and more accurate because the other approaches, having to sample the state space and train the NN multiple times due to their iterative structure, require more time to obtain good approximations. HJR performed worse than the others because, at each iteration, it relies on the NN trained at the previous iteration; this leads to an accumulation of the approximation error.

\begin{table}[!tbp]
\begin{center}
\caption{RMSE comparison for the 2D system.}
{\renewcommand\arraystretch{1.3}
\begin{tabular}{ cccc } 
\toprule
 & VBOC (ours) & AL & HJR\\ 
 \midrule
 \sc{Time} (s) & $13$ & $22$ & $117$\\ 
\sc{RMSE testing} (rad/s) & $0.0206$ & $0.0477$ & $0.3899$\\ 
 \bottomrule
\end{tabular}
}
\end{center}
\label{table:1prmse}
\end{table}

\subsection{Tests on a 4D system}
The tested system is now a double pendulum, an open kinematic chain with two swinging links and two planar revolute joints. The system has a 4D state space $x = [ q_1 \; q_2 \; \dot{q}_1 \; \dot{q}_2 ]^\T$ and a 2D control input $u = [ \tau_1 \; \tau_2 ]^\T$. \rev{The state constraints of each joint are the same of the previous test, the input constraints are now $\pm 10$ Nm. }


In these tests we decided to compare the RMSE evolutions for the three algorithms while learning. VBOC, as presented in Alg.~\ref{alg:vboc}, is not incremental, but we can easily make it incremental by alternating between data generation and training. At each iteration we computed a batch of \rev{$K = 1000$} data points. The AL algorithm is executed it on a grid with $60^4$ samples \rev{and batches of $1000$ points}. For HJR, \rev{to converge in a reasonable amount of time (58 minutes)}, we had to use a smaller number of samples: $20^4$. For VBOC and AL we have used NNs with 300 neurons in the hidden layer. For HJR instead, since the NN is here used inside the OCPs, we observed that it performed better using only 100 neurons due to the reduced complexity of the OCPs.

Fig.~\ref{fig:error2pevol} shows the RMSE evolutions over time, where we can see that VBOC reaches higher accuracies much faster than the other algorithms. At each iteration, HJR has to solve an OCP for each point and to train an accurate NN. Even if the OCPs have a shorter (1-step) horizon, the higher number of OCPs and the training of the intermediate NNs limit the number of samples that can be used and, consequently, the final accuracy. AL and VBOC point instead at minimizing the number of solved OCPs, which allows them to be more efficient.
Table~\ref{table:2prmse} reports the final RMSE error for the three algorithms, which is \rev{3.4} times larger for AL than for VBOC. Moreover, Table~\ref{table:2prmse} also reports the RMSE of VBOC on the training set, which is only slightly better than on the test set, highlighting a good generalization capability of the trained NN. 
Finally, looking at the cumulative error distribution in Fig.~\ref{fig:cumerror2p} we can see that VBOC not only resulted in a smaller average error (RMSE), but also in a lower number of errors above any given threshold.

\begin{figure}[!tbp]
    \centering
    \includegraphics[width=0.49\textwidth]{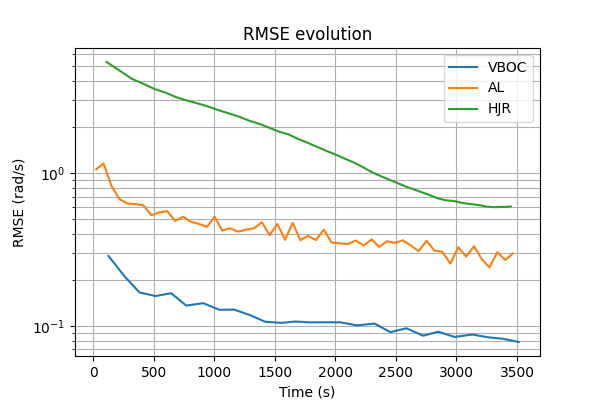}
    \caption{\rev{Comparison between the RMSE evolution for the 4D system.}}
    \label{fig:error2pevol}
\end{figure}

\begin{table}[!tbp]
\begin{center}
\caption{RMSE comparison for the 4D system.}
{\renewcommand\arraystretch{1.3}
\begin{tabular}{ cccc } 
 \toprule
 & VBOC (ours) & AL & HJR\\ 
 \midrule
 \sc{RMSE testing} (rad/s) & \rev{$0.0782\,$} & \rev{$0.2693\,$}& \rev{$0.6002\,$} \\ 
  \sc{RMSE training} (rad/s) & \rev{$0.0636\,$} & -- & -- \\ 
 \bottomrule
\end{tabular}
}
\end{center}
\label{table:2prmse}
\end{table}

\begin{figure}[!tbp]
    \centering
    \includegraphics[width=0.49\textwidth]{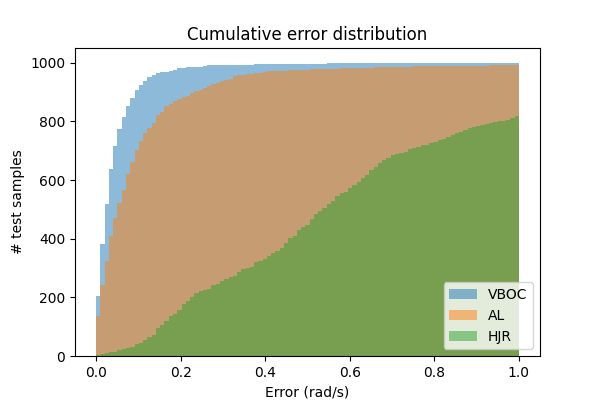}
    \caption{\rev{Cumulative error distribution for the 4D system at then end of the training. This plot shows how many test samples (y axis) obtained a prediction error (x axis) below a certain value.}}
    \label{fig:cumerror2p}
\end{figure}

\subsection{Tests on a 6D system}
To test the scalability of VBOC, we applied it also to a triple pendulum, which has a 6D state space $x = [ q_1 \; q_2 \; q_3 \; \dot{q}_1 \;\dot{q}_2 \; \dot{q}_3 ]^\T$ and a 3D control input $u = [ \tau_1 \; \tau_2  \; \tau_3 ]^\T$. \rev{The state and input constraints are the same as the previous test.} \qting{quote:RMSE_6D}{\rev{For this system we compare only VBOC and AL, since the curse of dimensionality of HJR resulted to be already too relevant. Indeed, HJR converged with an extremely high RMSE ($4.0$ rad/s) because we had to execute it on a grid of only $12^6$ points to make it converge within 6 hours.}}

\begin{figure}[!tbp]
    \centering
    \includegraphics[width=0.49\textwidth]{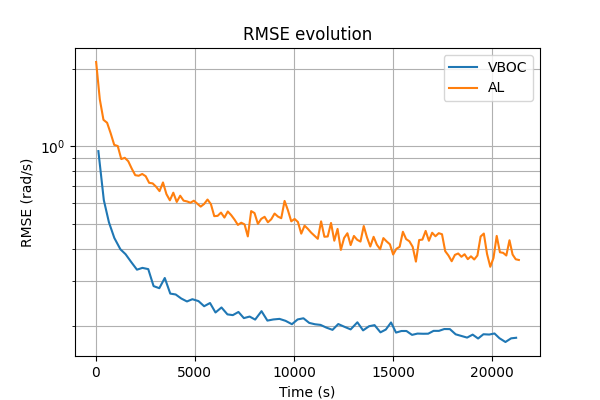}
    \caption{\rev{Comparison between the RMSE evolution for the 6D system.}}
    \label{fig:error3pevol}
\end{figure}

\begin{table}[!tbp]
\begin{center}
\caption{RMSE comparison for the 6D system.}
{\renewcommand\arraystretch{1.3}
\begin{tabular}{ ccc } 
 \toprule
 & VBOC (ours) & AL\\ 
 \midrule
 \sc{RMSE testing} (rad/s) & \rev{$0.1752\,$} & \rev{$0.3693\,$}\\ 
  \sc{RMSE training} (rad/s) & \rev{$0.0863\,$} & --  \\ 
 \bottomrule
\end{tabular}
}
\end{center}
\label{table:3prmse}
\end{table}

Table~\ref{table:3prmse} shows that \rev{the final RMSE obtained by VBOC is 2.2 times smaller than the one obtained with AL}. 
Fig.~\ref{fig:error3pevol} shows that the RMSE achieved by AL after \rev{6} hours is comparable to that achieved by VBOC after less than 30 minutes.
Even though the RMSE of VBOC on the training set is comparable to the one obtained for the 4D system, errors on the test set are larger, highlighting a lack of generalization. Smaller errors could be achieved by letting the algorithms run for more iterations, and/or using larger NN structures.



\section{Conclusions}
This paper presented a new algorithm (VBOC) for the approximation of the viability kernel of robot manipulators. Differently from state-of-the-art approaches, VBOC computes directly states on the boundary of the set, leading to more accurate results in a data-efficient manner. The set boundary has indeed a smaller dimension than the set itself, so VBOC scales more favourably than algorithms that explore the entire state space. VBOC is theoretically guaranteed to provide data on the boundary of the set, requiring only local optimality of the OCP solutions. Moreover, contrary to many state-of-the-art methods, VBOC does not need to rely on the ability of the OCP solver to correctly detect unfeasible problems, which makes it robust to numerical errors. \rev{Additionally, since the trained NN introduces some approximation errors due to its intrinsic inability to exactly represent the set, the OCPs complexity can be reduced by relaxing the solver tolerances to be only slightly more accurate than the expected error on the set approximation.}

\asia{Despite all of this, many issues still remain open. For instance, scalability is still a major concern, since the algorithmic complexity still scales exponentially. To address this, we could use customized NN structures that embed our prior knowledge on the shape of \Vs\ (e.g., we know that the maximum viable velocity is null when a joint is at its bound).}
Moreover, even if our tests have focused on joint-space constraints, we \rev{plan to} extend VBOC to Cartesian-space constraints, e.g., for obstacle avoidance. 
\qting{quote:mpc}{
Another interesting challenge is the use of the learned sets as terminal constraints in MPC (or safety filters in Reinforcement Learning). Since these sets are \emph{approximations} of \Vs , they are only \emph{approximately} control invariant, so recursive feasibility cannot be guaranteed in general. Therefore, we plan to investigate algorithmic approaches to use these sets, while maintaining strong guarantees of safety. \asia{Finally, while the theory in Section \ref{sec:vboc} holds for any differentiable system, our algorithm is specifically designed for star-convex sets; its extension to more generic cases is currently being investigated.}
}


\addtolength{\textheight}{-0cm}   


\bibliographystyle{IEEEtran}
\bibliography{IEEEabrv,references}

\end{document}